\documentclass[runningheads]{src/llncs}
\usepackage[T1]{fontenc}
\usepackage{graphicx}
\usepackage{booktabs}
\usepackage[misc]{ifsym}
\newcommand{\corr}{(\Letter)}
\usepackage{mwe}

\usepackage{subcaption}
\usepackage{hyperref}

\usepackage{amsmath,amssymb,amsfonts,mathtools}
\usepackage{bm}

\providecommand{\R}{\mathbb{R}}

\DeclareMathOperator{\Vol}{Vol}

\spnewtheorem{assumption}{Assumption}{\bfseries}{\itshape}

\begin{document}

\title{Across the Loss Landscape with Progressive Growth}
\toctitle{Across the Loss Landscape with Progressive Growth}

\titlerunning{Across the Loss Landscape with Progressive Growth}

\author{Paul Caillon\inst{1} \corr \and Christophe Cerisara\inst{2} \and Alexandre Allauzen\inst{1,3}}
\tocauthor{Paul Caillon, Christophe Cerisara, and Alexandre Allauzen}

\authorrunning{P. Caillon et al.}

\institute{MILES Team, LAMSADE, Université Paris Dauphine-PSL \email{\{name.surname\}@dauphine.psl.eu}
\and
LORIA CNRS, Campus Scientifique, rue du Jardin Botanique, 54500 Vandoeuvre-les-Nancy \email{christophe.cerisara@loria.fr}
\and
ESPCI-PSL}

\maketitle              

\begin{abstract}
Deep neural networks generalize well despite their highly nonconvex, overparameterized loss landscapes, a phenomenon often associated with the geometry of the minima found by stochastic optimization. We study how incremental “grow-and-optimize” strategies bias training toward flatter regions by viewing growth as progressive constraint relaxation. Starting from a low-dimensional submodel, we iteratively expand the trainable parameters by unlocking nested random subspaces while freezing the orthogonal complement at the network initialization, re-optimizing after each expansion until the full architecture is reached. Under standard local regularity conditions around non-degenerate minima, we prove that local sublevel sets are well approximated by ellipsoids and that basin accessibility under frozen constraints can be characterized by an explicit effective curvature in the frozen directions. This leads to an explanation of the bias: progressive growth increases the relative weight of wide basins and suppresses sharp ones through a volume effect induced by the frozen constraints. We empirically validate these predictions in controlled toy landscapes and in a realistic ResNet/CIFAR-100 setting and confirm that although progressive subspace growth reliably produces flatter solutions, curvature reductions do not universally translate into improved test performance, highlighting subtleties in the flatness–generalization connection. The code is available on \href{https://github.com/p0lcAi/Across-the-Loss-Landscape}{GitHub}.

\end{abstract}

\section{Introduction and Related Work}
\label{sec:intro_related}

Modern neural networks are routinely optimized in very high-dimensional parameter spaces and yet achieve strong generalization with relatively simple variants of stochastic gradient descent (SGD). 
A classical geometric narrative is that SGD tends to converge to \emph{flat} solutions, dating back at least to early discussions of flat minima as a robustness principle \cite{hochreiter1997flat}. Modern work has further connected this idea to stochasticity in training, viewing SGD as a noisy dynamical system whose long-run behavior resembles diffusion-like processes \cite{mandt2017stochastic,welling2011bayesian,smith2018bayesian}.
\newline
At the same time, the relation between flatness and generalization is now known to be subtle. Flatness is inherently local and strongly parameterization-dependent, and even sharp minima may generalize well under benign reparameterizations \cite{dinh2017sharp}. Empirically, sharpness-based quantities appear only partially predictive, often depending on the precise metric, architecture, and training regime considered, as highlighted by recent large-scale studies \cite{2023arXiv230207011A}. More broadly, the search for reliable generalization predictors has produced mixed conclusions: while some measures can be informative in carefully controlled settings \cite{jiang2019fantastic}, broader evaluations suggest that many proposed predictors fail to explain generalization consistently across realistic training variations \cite{gastpar2023fantastic}. These observations suggest a more cautious interpretation of flatness: rather than a universal explanation of generalization, it is more appropriate to view it as a \emph{geometric bias of training dynamics}, with strong setting-dependent  connection to test performance.
\newline
A complementary perspective studies the global structure of the loss landscape through connectivity and interpolation barriers. 
Here, we use the term \emph{basin} informally to denote a broad connected low-loss region of parameter space, rather than an isolated strict local minimum. 
Empirically, many trained solutions can be connected by low-loss curves, suggesting that what appear locally as distinct minima may in fact belong to larger connected regions \cite{garipov2018loss,draxler2018essentially}. Interpolation-based diagnostics have consequently become a useful practical tool for probing whether two solutions lie in the same broad basin and for quantifying barriers between them \cite{goodfellow2014qualitatively}. This line of work reinforces the idea that local curvature alone may be insufficient to characterize the optimization geometry that ultimately matters.
\newline
Because these geometric effects are shaped by optimization dynamics, the optimizer itself also matters. Momentum changes the effective dynamics and can accelerate movement along directions of low curvature \cite{Qian1999,Sutskever2013,Zhou2020momentum}. Adaptive methods such as Adam \cite{kingma2014adam} have likewise been argued to induce different implicit biases from SGD, sometimes converging to solutions with different sharpness properties \cite{Wilson2017marginal,Zhang2019adamsharp,Chen2020adamflatness,Barrett2020implicit}. At the same time, directly characterizing curvature in deep networks remains difficult: Hessian-based quantities are costly to estimate and inherit many of the same parameterization issues as flatness itself \cite{dinh2017sharp}. In practice, spectral proxies such as the top eigenvalue and the trace remain useful diagnostics, especially when combined with stochastic estimators based on Hessian-vector products \cite{yao2020pyhessian,ghorbani2019investigation,sagun2016eigenvalues}.
\newline
Alongside this literature on geometry and implicit bias, another family of works studies training procedures in which model capacity evolves over time. Function-preserving expansions and network morphism methods modify architectures while attempting to retain useful representations \cite{chen2015net2net}. Grow-and-prune strategies use expansion and sparsification to search for compact or efficient networks \cite{dai2019grow}. Related ideas also appear in neural architecture search and dynamic architecture adaptation, where capacity is explored progressively, either through continuous relaxations \cite{2018arXiv180609055L} or more explicit search-and-growth mechanisms \cite{wu2020firefly,verbockhaven2024growing}. In parallel, work on sparse and constrained training has shown that restricting optimization to a subset of parameters can still yield strong performance. Lottery-ticket-style results emphasize the existence of performant subnetworks within random initializations \cite{2018arXiv180303635F,2019arXiv190501067Z,2019arXiv191113299R}, while dynamic sparse training methods adapt sparsity patterns during learning \cite{evci2020rigging,mocanu2018scalable}. Together, these works show that training under evolving or structured parameter constraints is both practically relevant and algorithmically effective.
More directly related to our setting, \cite{caillon2021growing,caillon2024growing} provided initial theoretical and empirical evidence that growing neural networks tend to converge to flatter optima than their fixed-architecture counterparts. 
The theoretical arguments developed in these works rely on a stylized basin-volume view, in which lower-dimensional constraints are more likely to intersect wider basins.
\newline
The present work substantially extends these initial analyses by replacing this volume-based argument with a local geometric characterization of accessibility under progressive constraint relaxation. We study incremental model growth as a controlled perturbation of optimization: training starts in a low-dimensional submodel and successively unlocks \emph{nested} parameter subspaces until the full architecture is reached. Growth is used only during training, making it natural to view each stage as optimization on an affine slice where the active coordinates are trainable and the remaining coordinates are frozen at initialization. This perspective leads to a geometric question: which minima remain accessible as these constraints are relaxed?
\newline
Our central claim is that progressive growth induces an \emph{entropic} selection mechanism. A basin is more likely to be reached when it is compatible with many random frozen constraints, yielding an energy--entropy tradeoff reminiscent of ideas behind Entropy-SGD \cite{chaudhari2016entropy}, but arising here implicitly from the feasible set rather than from an explicit objective modification. Under local regularity assumptions, we derive accessibility probabilities that depend on both a location term, governed by the projection of a minimum onto frozen coordinates, and an entropy term, governed by an effective curvature in frozen directions through a Schur-complement reduction.
\newline
This viewpoint also guides our empirical protocol. Rather than relying solely on scalar flatness measures, we evaluate whether growth preserves or changes basin membership across stages using interpolation barriers between pre- and post-expansion solutions, together with retention and leakage metrics that directly reflect the theory. We first validate the mechanism in controlled toy problems where minima and curvatures are known, and then study the same phenomena in realistic deep learning settings, focusing on ResNet training on CIFAR-100. In this way, the paper aims to connect a theory of accessibility under constraint relaxation with a practical experimental methodology for studying growth-induced geometric bias.
We emphasize that our goal is explanatory rather than algorithmic. Progressive growth is studied as a controlled intervention on the feasible training set, not as a plug-in recipe for improving test accuracy. Accordingly, our empirical evaluation focuses on transition geometry in addition to final predictive performance. We also use internal controls to separate progressive constraint relaxation from simpler forms of delayed capacity release.
\section{Theoretical Analysis: Progressive Growth Biases Optimization Toward Flat Regions}
\label{sec:theory}

In this section we develop a local geometric explanation for why progressive growth tends to favor flatter minima. The key observation is that progressive growth can be viewed as a sequence of constrained optimization problems: at each stage, only a subset of parameters is trainable, while the remaining coordinates are frozen at their initialization values. Training therefore evolves on a sequence of affine slices of increasing dimension rather than in the full parameter space from the start. A minimum can be reached at a given stage only if the current slice passes sufficiently close to it. Our main result is that this compatibility condition is easier to satisfy for minima that are flatter along the frozen directions, so that progressive growth statistically favors such minima.

The analysis is intentionally local. We do not assume that optimization probabilities are governed by global basin volume, nor do we claim that flatness is a universal explanation of generalization. Instead, we isolate a geometric mechanism specific to progressive growth: under partial freezing, broad minima remain compatible with a larger set of frozen constraints and therefore remain reachable more often throughout the growth process. For readability, we present the main statements and their interpretation here and defer the proofs to the Appendix~\ref{app:theory_proofs}.

\subsection{Setup: progressive growth as optimization on affine slices}

Let $f:\R^d\to\R$ denote the empirical risk, and let $\theta^0\in\R^d$ be a random initialization. We consider a stage-wise growth procedure
$$
0 < K_0 < K_1 < \cdots < K_T = d,
$$
where at stage $t$ only a $K_t$-dimensional subspace is trainable. For theoretical clarity, we model this active subspace as random. Let $Q\in O(d)$ be Haar-distributed, and write
$$
Q = [U_t\;V_t],
$$
where $U_t\in\R^{d\times K_t}$ spans the active subspace and $V_t\in\R^{d\times p_t}$ spans the frozen complement, with $p_t=d-K_t$. The feasible set at stage $t$ is then
$$
\mathcal{A}_t
=
\theta^0 + \mathrm{span}(U_t)
=
\{\theta\in\R^d : V_t^\top(\theta-\theta^0)=0\}.
$$
The corresponding idealized stage-wise solution is
\begin{equation}
\theta_t \in \arg\min_{\theta\in\mathcal{A}_t} f(\theta).
\label{eq:stagewise-min}
\end{equation}

This formulation makes the geometry explicit: at stage $t$, optimization can move only inside the affine slice $\mathcal{A}_t$, and growth gradually relaxes this constraint by enlarging the active subspace.

\subsection{Local minima and accessibility}

Assume that $f$ has isolated local minima
$
\{\theta_i^\star\}_{i=1}^m .
$

Fix one such minimum $\theta_i^\star$. For a small tolerance $\varepsilon>0$, define its local basin by
$$
\mathcal{L}_{i,\varepsilon}
=
\{\theta : f(\theta)\le f(\theta_i^\star)+\varepsilon\}
\cap
B(\theta_i^\star,r_i),
$$
where $r_i>0$ is chosen so that these neighborhoods do not overlap. We say that basin $i$ is \emph{$\varepsilon$-accessible} at stage $t$ if
$$
\mathcal{A}_t \cap \mathcal{L}_{i,\varepsilon}\neq\emptyset.
$$
Accessibility simply means that the current growth stage still allows optimization to come sufficiently close to that minimum. If the frozen coordinates force the affine slice away from the basin, then that minimum is unavailable at that stage no matter how well optimization is performed inside the slice.

\subsection{Local regularity and quadratic approximation}

Our analysis uses a standard local quadratic model around each minimum.

\begin{assumption}[Local regularity]
\label{ass:local}
For each minimum $\theta_i^\star$,
$$
\nabla f(\theta_i^\star)=0,
\qquad
H_i := \nabla^2 f(\theta_i^\star)\succ 0.
$$
Moreover, the Hessian is locally Lipschitz: there exist $r_i>0$ and $\rho_i>0$ such that, for all $\theta,\theta'\in B(\theta_i^\star,r_i)$,
$$
\|\nabla^2 f(\theta)-\nabla^2 f(\theta')\|
\le
\rho_i\|\theta-\theta'\|.
$$
\end{assumption}

This positive-definite Hessian assumption should be understood as a local
regularity model rather than as a literal description of modern
overparameterized networks. Deep networks typically contain exact or approximate
symmetries, which induce flat or nearly flat directions and can turn isolated
minima into local minimum manifolds. In such settings, the present analysis can
be interpreted after quotienting out exact symmetries, or equivalently by
restricting attention to the normal directions of a local minimum manifold. The
Schur-complement mechanism then applies to the non-flat normal block, while
tangent directions contribute neutral volume factors. A complete treatment of
degenerate minimum manifolds is beyond the scope of this work, but the
transition diagnostics used in the experiments below are designed precisely to avoid relying on isolated-minimum structure alone.

Under this assumption, the loss near $\theta_i^\star$ behaves like a quadratic bowl with curvature matrix $H_i$. The local basin can thus be approximated by an ellipsoid.

\begin{proposition}[Local ellipsoidal approximation]
\label{prop:local-ellipsoid}
Under Assumption~\ref{ass:local}, there exist constants $c_{1,i},c_{2,i}>0$ and $\varepsilon_i^{\max}>0$ such that, for all $0<\varepsilon\le \varepsilon_i^{\max}$,
$$
\left\{
\theta_i^\star+\delta :
\frac12 \delta^\top H_i\delta \le c_{1,i}\varepsilon
\right\}
\subset
\mathcal{L}_{i,\varepsilon}
\subset
\left\{
\theta_i^\star+\delta :
\frac12 \delta^\top H_i\delta \le c_{2,i}\varepsilon
\right\}.
$$
\end{proposition}

Thus, up to constant-factor distortions, the local basin is governed by the Hessian at the minimum. This is the first place where curvature enters the theory.

\subsection{Frozen coordinates induce an effective curvature}

We now analyze the effect of the affine constraint $\mathcal{A}_t$. Write local coordinates around $\theta_i^\star$ in the basis $[U_t\;V_t]$:
$
\theta
=
\theta_i^\star + U_t a + V_t b,
\qquad
a\in\R^{K_t},\;\; b\in\R^{p_t}.$

Here $a$ represents the active coordinates and $b$ the frozen ones. Define
$$
g_{i,t}(a,b)
:=
f(\theta_i^\star+U_t a+V_t b).
$$
For a fixed frozen offset $b$, the best loss attainable after optimizing the active coordinates is
$
\varphi_{i,t}(b)
=
\min_{a}
\bigl(g_{i,t}(a,b)-f(\theta_i^\star)\bigr).
$

The frozen coordinates imposed by the initialization are
$
b_{i,t}^0 := V_t^\top(\theta^0-\theta_i^\star).
$

Thus basin $i$ is $\varepsilon$-accessible at stage $t$ exactly when
$
\varphi_{i,t}(b_{i,t}^0)\le \varepsilon.
$

The reduced function $\varphi_{i,t}$ is again locally quadratic, but with a different curvature matrix.

\begin{theorem}[Effective curvature under frozen constraints]
\label{thm:schur}
Under Assumption~\ref{ass:local}, for $b$ in a neighborhood of $0$,
$
\varphi_{i,t}(b)
=
\frac12\, b^\top \Sigma_{i,t} b + O(\|b\|^3),
$
where
$
\Sigma_{i,t}
=
H^{(t)}_{i,bb}
-
H^{(t)}_{i,ba}(H^{(t)}_{i,aa})^{-1}H^{(t)}_{i,ab}
$
is the Schur complement of the Hessian block matrix in the coordinates $(a,b)$.
\end{theorem}

This matrix $\Sigma_{i,t}$ is the central geometric object in the analysis. It measures how costly it is to keep the frozen coordinates away from the minimum after re-optimizing all active coordinates. When $\Sigma_{i,t}$ is small, the loss rises slowly as the frozen coordinates are perturbed, so many frozen configurations remain compatible with that minimum. In this sense, $\Sigma_{i,t}$ captures the effective flatness of the basin along the frozen directions.

\subsection{Compatibility volume and local accessibility law}

Define the compatibility set
$
\mathcal{B}_{i,t}(\varepsilon)
=
\{b\in\R^{p_t} : \varphi_{i,t}(b)\le \varepsilon\}.
$

By Theorem~\ref{thm:schur}, this set is locally approximated by the ellipsoid
$$
\left\{
b :
\frac12\,b^\top\Sigma_{i,t}b \le \varepsilon
\right\},
$$
whose volume scales like $\det(\Sigma_{i,t})^{-1/2}$. This already shows the geometric mechanism: flatter minima in the frozen directions admit a larger set of compatible frozen coordinates.

To convert this volume statement into a probability statement, one must also account for how the frozen offset $b_{i,t}^0$ is distributed. Assume that, conditional on the chosen subspace $(U_t,V_t)$, the random variable $b_{i,t}^0$ has a continuous density $p_{i,t}$ near $0$. Then the accessibility probability obeys the following local law.

\begin{theorem}[Local accessibility law]
\label{thm:accessibility}
Under Assumption~\ref{ass:local}, assume that the frozen offset
$
b_{i,t}^0 = V_t^\top(\theta^0-\theta_i^\star)
$
has a continuous density $p_{i,t}$ in a neighborhood of $0$, conditional on $(U_t,V_t)$. Then, as $\varepsilon\to 0$,
$$
\mathbb{P}\!\left(\mathcal{A}_t\cap\mathcal{L}_{i,\varepsilon}\neq\emptyset \,\middle|\, U_t,V_t\right)
=
p_{i,t}(0)\,
\kappa_{p_t}(2\varepsilon)^{p_t/2}
\det(\Sigma_{i,t})^{-1/2}
\bigl(1+o(1)\bigr),
$$
where $\kappa_{p_t}$ denotes the volume of the unit ball in $\R^{p_t}$.
\end{theorem}

The accessibility probability therefore decomposes into two factors. The density term $p_{i,t}(0)$ measures how likely the frozen coordinates are to place the affine slice near the minimum at all, while the determinant term $\det(\Sigma_{i,t})^{-1/2}$ measures how large the set of compatible frozen coordinates is once the slice is nearby. The second term is the flatness contribution: all else being equal, minima with smaller effective frozen-direction curvature are accessible under a larger set of frozen constraints.

In the common case of isotropic Gaussian initialization, the density term decreases with the squared distance between the frozen projection of the minimum and the initialization, while the determinant term rewards broad minima. Progressive growth therefore creates an energy--entropy tradeoff: minima that are both close to the initialization in frozen coordinates and broad in those directions are favored.

\subsection{Interpretation of the Theoretical Results}

Theorems~\ref{thm:schur} and~\ref{thm:accessibility} do not claim that stage-wise gradient descent always converges to the globally flattest minimum. What they show is more directly tied to progressive growth: under partial freezing, flatter minima remain accessible under a larger fraction of constraints. This means that early growth stages impose the most severe restrictions, since many coordinates are still frozen. At those stages, narrow minima are fragile: a modest mismatch in the frozen coordinates is enough to make them inaccessible. Broad minima, by contrast, tolerate a wider range of frozen values and therefore remain reachable even under strong restrictions. As the model grows and the constraints are gradually relaxed, these broad minima are more likely to persist across stages. This is the precise sense in which progressive growth biases optimization toward flatter minima.

The theory yields three concrete predictions. First, flatter regions should appear earlier along the growth trajectory, because broad minima remain accessible under stronger constraints than narrow ones. Second, once the optimization trajectory enters such a broad basin, subsequent growth steps should tend to preserve it rather than force optimization to leave it, leading to strong stage-to-stage retention and limited immediate degradation after expansion. Third, if consecutive stages remain within the same broad basin, then the corresponding solutions should be connected by low-loss paths, so interpolation between stage solutions should exhibit small barriers. In the next section, we test these predictions using curvature diagnostics, retention and leakage metrics, and interpolation barriers between successive stage-wise solutions.
\section{Empirical Investigation}
\label{sec:empirical}

In this section, we evaluate the predicted geometric bias first in controlled toy landscapes,
where basin geometry is known by construction, and then in ResNet/CIFAR-100
experiments, where we test whether the same qualitative signatures persist in
realistic deep networks. Additional experimental details and results are provided in the Appendix~\ref{app:exp_details}.

\subsection{Toy validation of the growth-induced flatness bias}
\label{sec:toy}

We first validate the theoretical mechanism in a controlled setting where basin geometry is fully known. Rather than introducing optimizer-specific effects, these toy experiments isolate the accessibility bias induced by progressive growth itself. We consider synthetic quadratic basins in dimension $d=100$, each defined by a center $\mu_i$ and a positive-definite curvature matrix $H_i$, and evaluate how often a basin is selected under nested frozen-coordinate constraints. For each trial, we sample a random initialization and a random orthogonal basis, apply a growth schedule with stage counts $S\in\{1,5,10,20,50,100\}$, and record the selected basin. Repeating this procedure over 500 trials yields an empirical selection distribution over basin families. The results are shown in Fig.~\ref{fig:toy_main}.

\begin{figure*}[h]
    \centering
    \begin{subfigure}[t]{0.46\textwidth}
        \centering
\includegraphics[width=\textwidth]{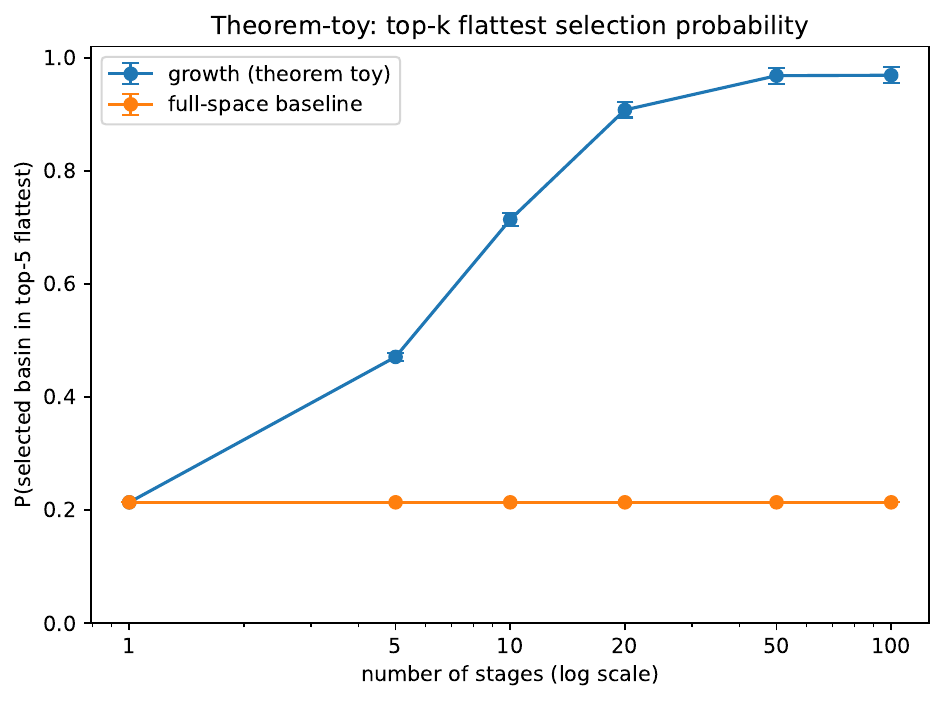}
        \caption{}
        \label{fig:toy_mult_topk}
    \end{subfigure}
    \hfill
    \begin{subfigure}[t]{0.46\textwidth}
        \centering
\includegraphics[width=\textwidth]{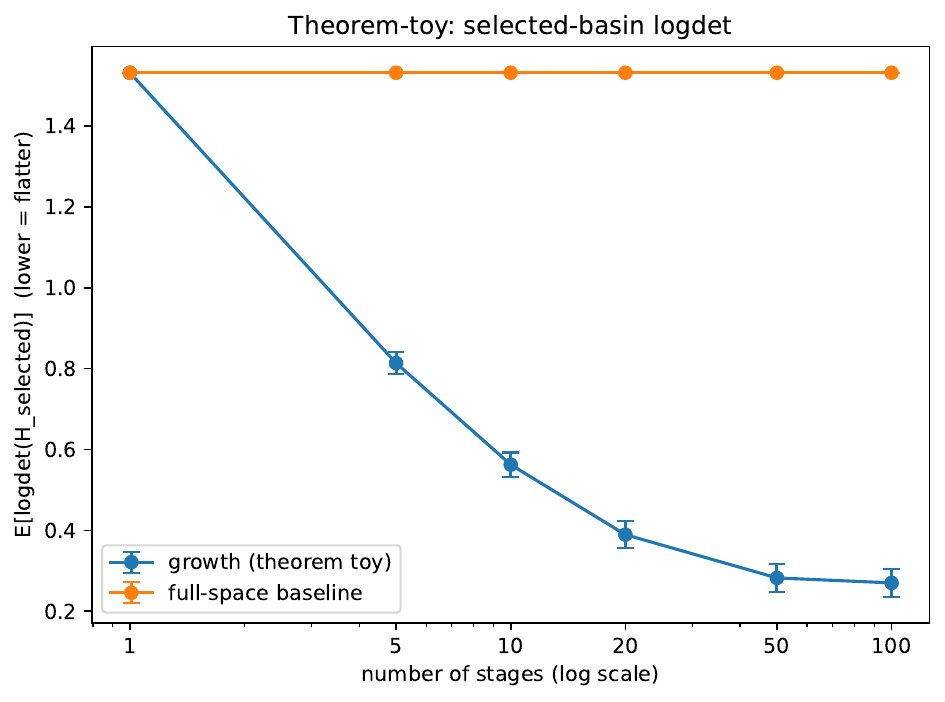}
        \caption{}
        \label{fig:toy_mult_logdet}
    \end{subfigure}
    \begin{subfigure}[t]{0.46\textwidth}
        \centering
\includegraphics[width=\textwidth]{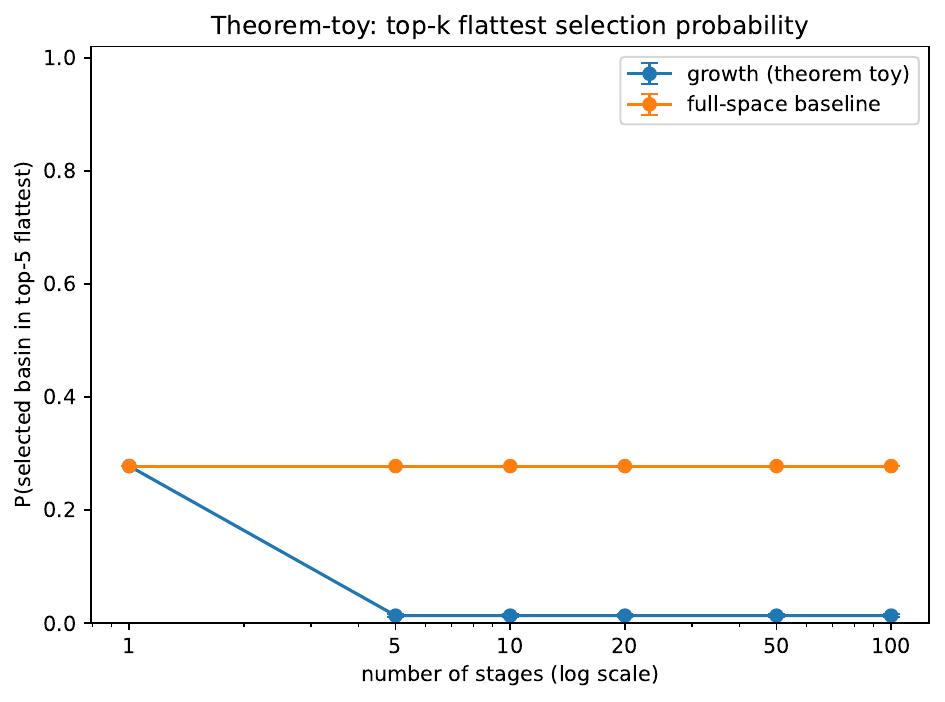}
        \caption{}
        \label{fig:toy_trade_topk}
    \end{subfigure}
    \hfill
    \begin{subfigure}[t]{0.46\textwidth}
        \centering
\includegraphics[width=\textwidth]{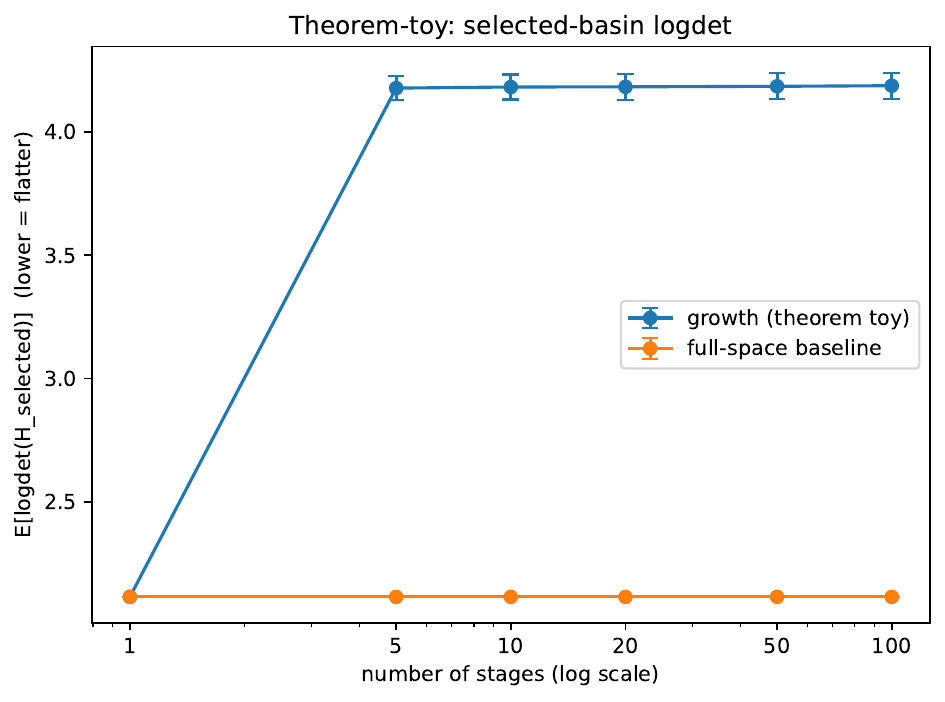}
        \caption{}
\label{fig:toy_trade_logdet}
    \end{subfigure}

    \caption{\textbf{Toy validation of the growth bias under controlled basin geometry.}
    Top row: multiplicity regime. Bottom row: energy--entropy trade-off regime.
    Left column: probability of selecting one of the top-$5$ flattest basins as a function of the number of growth stages.
    Right column: expected curvature of the selected basin, measured by $\log\det(H)$.
    In the multiplicity regime, progressive growth increasingly selects flatter basins. In the trade-off regime, this trend reverses when sharp basins are closer to initialization, illustrating the energy--entropy competition predicted by the theory.}
    \label{fig:toy_main}
\end{figure*}

We study two regimes. In the \emph{multiplicity} regime, flat basins are more numerous while basin locations are balanced relative to initialization, so the dominant effect should come from curvature. In the \emph{energy--entropy trade-off} regime, sharp basins are placed closer to initialization, while flat basins retain the entropy advantage associated with larger compatibility volume. The first regime therefore tests the pure flatness effect predicted by the theory, whereas the second tests whether this effect can be overcome by a sufficiently favorable location term.

In the multiplicity regime, increasing the number of growth stages strongly shifts selection toward flatter basins. The probability of selecting one of the top-$5$ flattest minima increases monotonically with the number of stages (Fig.~\ref{fig:toy_mult_topk}), while the expected curvature of the selected basin, measured by $\log\det(H)$, decreases accordingly (Fig.~\ref{fig:toy_mult_logdet}). This is exactly the behavior predicted by the accessibility law: repeated freezing and release acts as a geometric filter that disproportionately preserves broad minima.

In contrast, the trade-off regime shows that the flatness bias is not unconditional. When sharp basins are substantially closer to initialization, the trend reverses: progressive growth increasingly selects nearby sharp basins rather than distant flat ones. This is visible both in the collapse of the top-$5$ flat-basin selection probability (Fig.~\ref{fig:toy_trade_topk}) and in the corresponding increase in the curvature of the selected basin (Fig.~\ref{fig:toy_trade_logdet}). This reversal is consistent with the theory, which predicts an energy--entropy competition between distance-to-initialization and compatibility volume.
Taken together, these toy experiments validate the central theoretical claim. Progressive growth does induce a genuine bias toward flatter minima, but this bias operates through accessibility under frozen constraints rather than through an unconditional preference for flatness. When basin locations are comparable, the flatness term dominates and growth amplifies selection toward broad minima. When location and curvature conflict, the observed selection follows the predicted trade-off.

\subsection{Does the bias translate to deep learning models?}
\label{sec:resnet}

The toy experiments validate the mechanism in a controlled setting where basin geometry is known by construction. We now ask whether the same qualitative bias remains visible in realistic deep networks. To this end, we study progressive growth in ResNet models trained on CIFAR-100. 
Since the theory concerns accessibility under frozen constraints rather than predictive performance per se, we focus not only on final accuracy but also on transition-level quantities. For two consecutive stage solutions, we measure the \emph{interpolation barrier}, defined as the maximal excess loss along the linear interpolation between the two endpoints; \emph{retention}, which indicates whether the post-expansion solution remains in the same low-barrier regime as the pre-expansion one; \emph{leakage}, which quantifies how strongly the transition departs from this retained regime when the barrier is non-negligible; and restricted curvature on the active and newly released subspaces. Together, these metrics test whether progressive growth preserves the current basin across expansions and whether the newly released directions are geometrically broad.
Their formal definitions are given in the Appendix~\ref{app:transition_metrics}.

The main results are summarized in Table~\ref{tab:resnet_main}. The primary baseline is a standard ResNet-18 trained without freezing or growth: all parameters are active from initialization, with the same data split, optimizer, training budget, batch sizes, augmentation, and learning-rate schedule as the growth runs. From the perspective of final predictive performance, progressive growth is not beneficial in this setting. The full-model baseline reaches the best test accuracy, while all growth schedules perform slightly worse. This trend is shown more explicitly in Fig.~\ref{fig:resnet_perf_patience} (left). With validation-accuracy-based growth, the final test accuracy decreases from $74.99\%$ at $S=3$ to $74.94\%$ at $S=5$ and $73.94\%$ at $S=10$. A similar trend appears for training-loss-based growth, which starts slightly higher at $75.54\%$ for $S=3$ but then drops to $74.40\%$ at $S=5$ and $73.37\%$ at $S=10$. Thus, progressive growth does not improve final generalization when judged solely by the final test metric.

However, the geometric behavior of the resulting trajectories is much more structured. For moderate schedules, consecutive stage solutions often remain in the same broad low-loss region. This is especially clear for the mild training-loss schedule with $S=3$, for which the mean interpolation barrier is essentially zero and the mean retention is $1.0$, indicating near-perfect stage-to-stage preservation. Validation-accuracy-based growth is also consistently stable: its mean barriers remain small across all stage counts, and its mean retention is high, increasing from $0.83$ at $S=3$ to $0.92$ at $S=5$ and $0.96$ at $S=10$. Thus, although growth does not improve the final predictor, moderate schedules often preserve the current broad region rather than forcing optimization into a different basin.

\begin{figure}[h]
    \centering

    \begin{subfigure}[t]{0.35\linewidth}
        \centering
        \includegraphics[width=\linewidth]{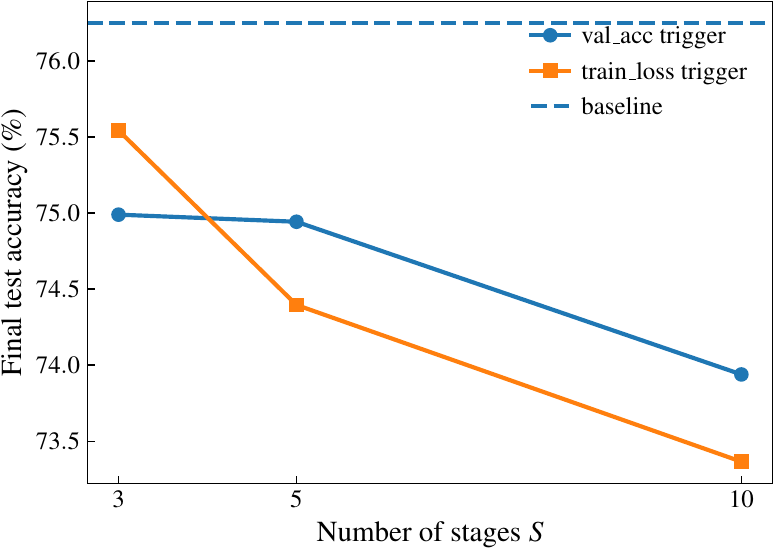}
        \caption{}
        \label{fig:resnet_testacc}
    \end{subfigure}
    \hfill
    \begin{subfigure}[t]{0.62\linewidth}
        \centering
        \includegraphics[width=\linewidth]{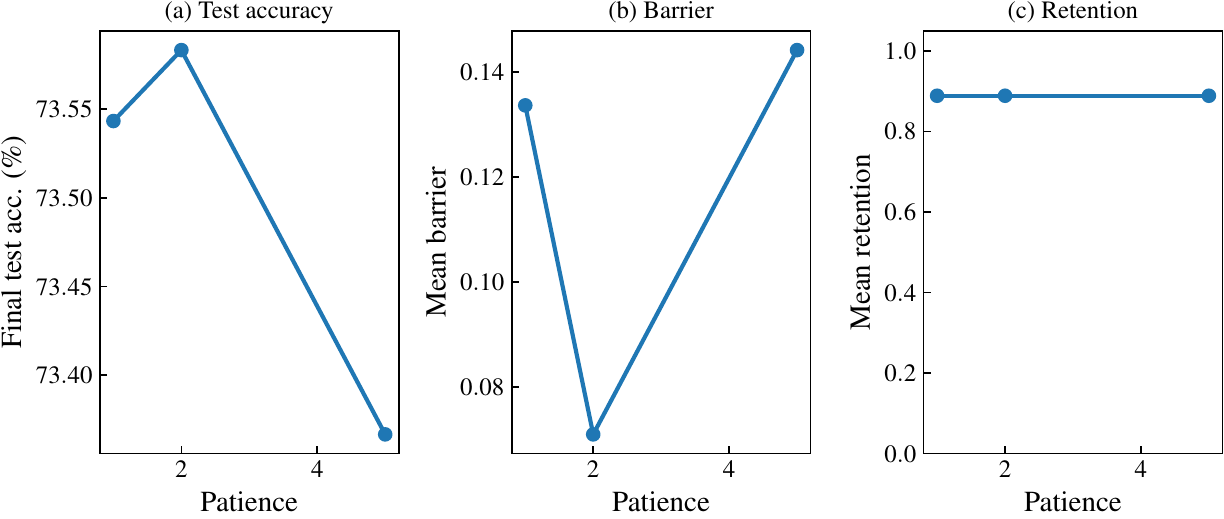}
        \caption{}
        \label{fig:resnet_patience}
    \end{subfigure}

    \caption{\textbf{Predictive performance and patience sensitivity under progressive growth.}
    \textbf{Left:} Final test accuracy as a function of the number of growth stages. Increasing the number of stages does not improve performance over the full-model baseline and generally reduces final test accuracy.
    \textbf{Right:} Patience ablation for the aggressive $S=10$ training-loss schedule. Patience has only a mild effect on final accuracy, but changes the mean transition barrier more noticeably, indicating that it mainly affects stage-to-stage geometric stability.}
    \label{fig:resnet_perf_patience}
\end{figure}

The contrast between the two triggers becomes more pronounced as the number of stages increases. Validation-accuracy-based growth remains geometrically mild even at larger $S$, with barriers staying low and retention remaining high. By contrast, training-loss-based growth becomes markedly less stable once the schedule is more fragmented. The mean barrier increases sharply from $0.00$ at $S=3$ to $0.125$ at $S=5$ and $0.144$ at $S=10$, while retention drops from $1.0$ to $0.75$ and then partially recovers to $0.89$. This suggests that training-loss triggering is more prone to letting the optimization drift away from the current solution once new degrees of freedom are released, especially when growth is split into many small stages. These trends are summarized in Fig.~\ref{fig:resnet_geometry}(a,b).

\begin{table}[h]
\centering
\caption{\textbf{ResNet-18 on CIFAR-100.} Progressive growth does not improve final test accuracy over the full-model baseline, but it induces clear geometric effects. The full-model baseline uses the exact same ResNet-18 architecture and training
budget, but without any freezing or growth: all parameters are active from
initialization. We report mean $\pm$ standard deviation over 3 seeds. Barrier values are reported in units of $10^{-2}$. For moderate schedules, consecutive stages are often strongly retained and barriers remain small. Across all schedules, the newly released directions are substantially flatter than the already active ones, as shown by the ratio $\lambda_{\max}^{\text{new}}/\lambda_{\max}^{\text{act}}<1$.}
\resizebox{\linewidth}{!}{
\begin{tabular}{llccccc}
\toprule
Method & Trigger & $S$ & Test acc. (\%) $\uparrow$ & Barrier ($\times 10^{-2}$) $\downarrow$ & Retention $\uparrow$ & $\lambda_{\max}^{\text{new}}/\lambda_{\max}^{\text{act}} \downarrow$ \\
\midrule
Full ResNet-18 &  & 1 & $76.25 \pm 0.23$ & -- & -- & -- \\
\midrule
Growth & val\_acc    & 3  & $74.99 \pm 0.48$ & $3.37 \pm 3.44$ & $0.833 \pm 0.289$ & $0.313 \pm 0.056$ \\
Growth & val\_acc    & 5  & $74.94 \pm 0.70$ & $3.82 \pm 5.93$ & $0.917 \pm 0.144$ & $0.201 \pm 0.014$ \\
Growth & val\_acc    & 10 & $73.94 \pm 0.86$ & $2.67 \pm 4.26$ & $0.963 \pm 0.064$ & $0.0896 \pm 0.0058$ \\
\midrule
Growth & train\_loss & 3  & $75.54 \pm 0.07$ & $0.00 \pm 0.00$ & $1.000 \pm 0.000$ & $0.305 \pm 0.015$ \\
Growth & train\_loss & 5  & $74.40 \pm 0.58$ & $12.47 \pm 4.16$ & $0.750 \pm 0.000$ & $0.174 \pm 0.014$ \\
Growth & train\_loss & 10 & $73.37 \pm 0.23$ & $14.42 \pm 1.64$ & $0.889 \pm 0.000$ & $0.0859 \pm 0.0078$ \\
\bottomrule
\end{tabular}}
\label{tab:resnet_main}
\end{table}

To further test whether these transition-level quantities are genuinely informative, we performed a post-hoc analysis at the level of individual growth steps across the main ResNet experiments. The results are shown in Fig.~\ref{fig:resnet_posthoc_transition}. Each point corresponds to a single stage transition from the main ResNet runs. We observe that transitions with larger interpolation barriers also tend to exhibit larger absolute endpoint gaps after expansion, i.e., larger absolute differences between the losses at the pre- and post-expansion endpoints of the transition, with a Spearman correlation of approximately $0.31$. Likewise, transitions with larger leakage tend to produce larger endpoint changes, with a Spearman correlation of approximately $0.38$. These trends indicate that the proposed transition metrics are not merely descriptive geometric summaries: they are predictive of how strongly a growth step perturbs the current solution in practice.
\begin{figure}[t]
    \centering
    \includegraphics[width=\linewidth]{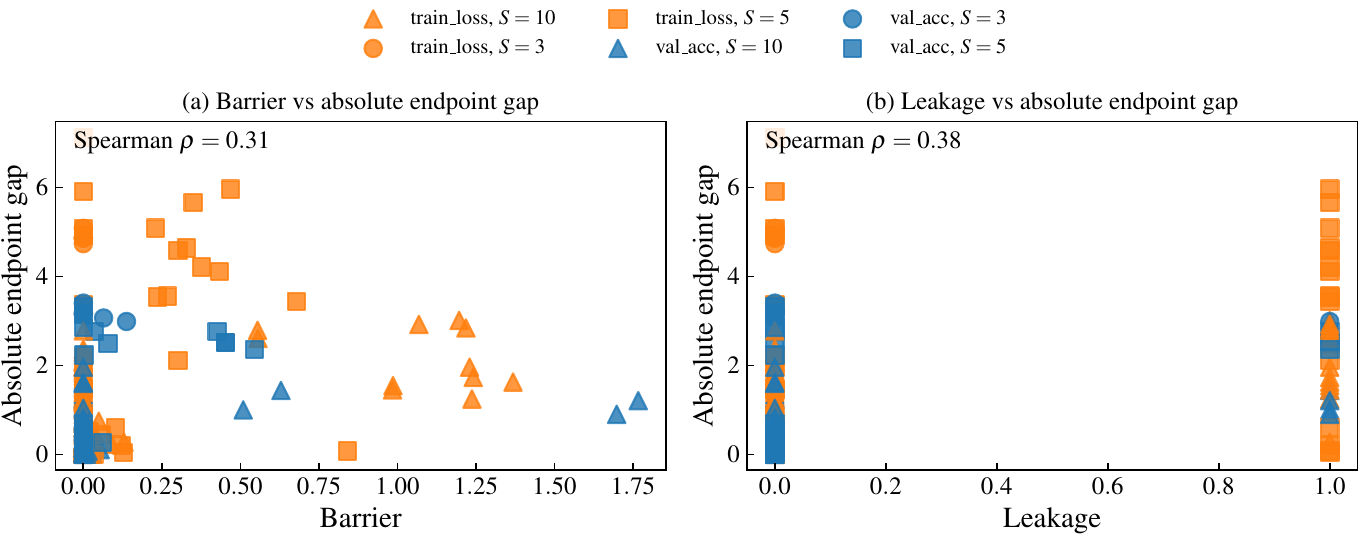}
    \caption{\textbf{Analysis of transition quality in the ResNet experiments.}
    Each point corresponds to a single stage transition. Left: interpolation barrier versus absolute endpoint gap. Right: leakage versus absolute endpoint gap. In both cases, more disruptive transitions are associated with larger endpoint changes after expansion, showing that the proposed transition-level metrics are predictive of how strongly a growth step perturbs the current solution.}
    \label{fig:resnet_posthoc_transition}
\end{figure}

A second robust pattern concerns curvature. Across all schedules and both triggers, the newly released directions are systematically flatter than the already active ones. This is visible in Fig.~\ref{fig:resnet_geometry}(c), where the ratio $\lambda_{\max}^{\mathrm{new}}/\lambda_{\max}^{\mathrm{act}}$ remains far below $1$ in every configuration. For both triggers, this ratio is about $0.31$ at $S=3$, drops to roughly $0.20$ for validation-accuracy growth and $0.17$ for training-loss growth at $S=5$, and reaches about $0.09$ at $S=10$. Thus, when growth unlocks new parameters, it typically reveals directions of substantially lower local curvature than those already optimized. This observation is highly consistent with the constrained-slice view developed in Section~\ref{sec:theory}: the newly accessible degrees of freedom are precisely those that remain compatible with the current solution under constraint relaxation, and in practice they correspond to relatively broad directions of the loss landscape.

\begin{figure*}[t]
    \centering
    \includegraphics[width=0.96\linewidth]{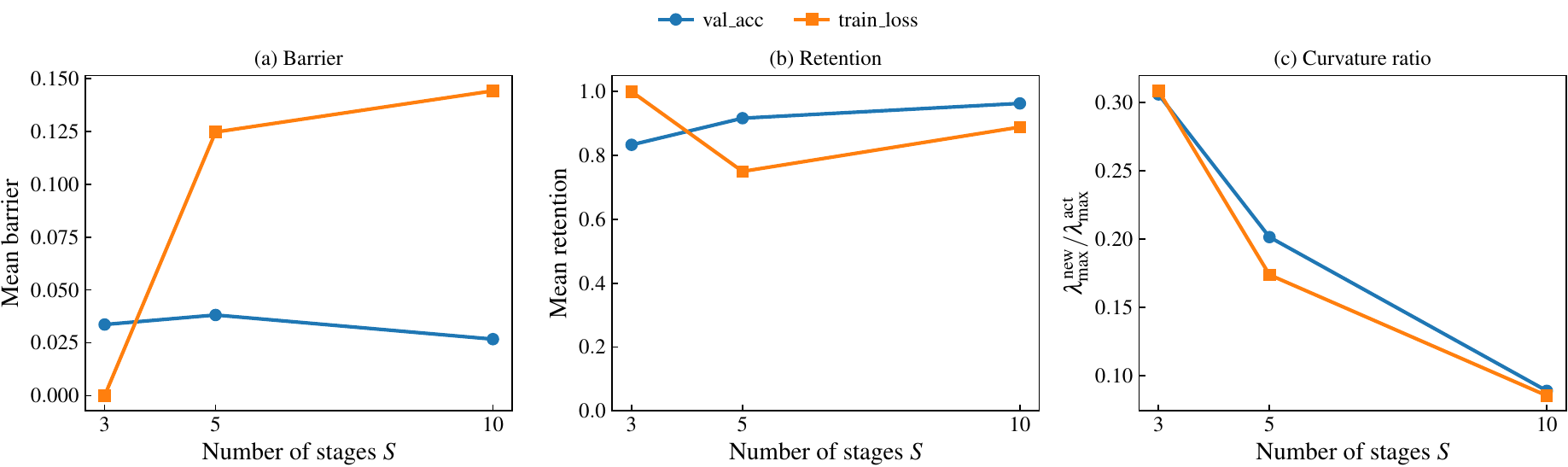}
    \caption{\textbf{Transition geometry under progressive growth in ResNet-18 on CIFAR-100.}
    (a) Mean interpolation barrier between consecutive stage solutions. (b) Mean retention across stage transitions. (c) Ratio between the top Hessian eigenvalue on the newly released subspace and that on the already active subspace. Validation-accuracy-based growth yields consistently mild transitions with low barriers, while training-loss-based growth becomes less stable as the number of stages increases. Across all schedules, the newly released directions are substantially flatter than the already active ones, as shown by $\lambda_{\max}^{\mathrm{new}} / \lambda_{\max}^{\mathrm{act}} \ll 1$.}
    \label{fig:resnet_geometry}
\end{figure*}

To better understand the most aggressive regime, we also performed a patience ablation for the training-loss schedule at $S=10$, shown in Fig.~\ref{fig:resnet_perf_patience} (right). Here, patience denotes the number of consecutive epochs without improvement of the stage metric before triggering the next growth step. The final test accuracy changes only mildly across patience values, whereas the mean transition barrier varies more noticeably. In particular, patience $2$ yields the lowest mean barrier, while patience $1$ and $5$ lead to larger barriers despite very similar final accuracies. Retention remains essentially unchanged. This suggests that patience mainly modulates stage-to-stage geometric stability rather than final predictive performance. In other words, the timing of expansion affects how smoothly the optimization trajectory adapts to newly released parameters, even when the final accuracy remains largely unchanged.
Taken together, these results provide a nuanced answer to the motivating question. The bias predicted by the theory does translate to deep learning models, but primarily as a geometric effect rather than as a direct improvement in final accuracy. Progressive growth leaves a clear footprint on optimization trajectories: moderate schedules preserve consecutive stage solutions, interpolation barriers remain small, and newly released subspaces are consistently flatter than the already active ones. In this sense, what translates to deep networks is not a guaranteed gain in predictive performance, but a systematic tendency for growth to preserve broad regions and to reveal comparatively flat directions as the trainable subspace expands.

\subsection{Delayed release versus progressive relaxation}

The experiments above compare full-model training with several progressive
growth schedules. To isolate the role of gradual constraint relaxation, we
consider two internal controls. The first one, fixed-subspace training, keeps the
initial trainable subspace fixed throughout optimization. The second one, one-shot unfreezing, starts from the same type of constrained
submodel as progressive growth and uses the same growth criterion. When the
criterion triggers expansion, it releases all remaining parameters in a single
step rather than unlocking them progressively. Thus, one-shot unfreezing controls for delayed capacity release while
removing the gradual relaxation mechanism.

\begin{table}[t]
\centering
\caption{
Internal controls on ResNet-18/CIFAR-100. The full baseline uses the same
architecture, optimizer, data split, augmentation, epoch budget, and
learning-rate schedule, but all parameters are active from initialization.
Fixed-subspace training keeps the initial trainable subspace fixed throughout
optimization. 
One-shot unfreezing starts from the same initial constrained
submodel as progressive growth and follows the same validation-accuracy
trigger, but releases all remaining parameters in a
single transition instead of unlocking them stage by stage. These controls
separate the effect of delayed capacity release from gradual constraint
relaxation.
}
\label{tab:resnet_internal_controls}
\resizebox{\linewidth}{!}{
\begin{tabular}{lccccc}
\toprule
Method
& Test acc. (\%) $\uparrow$
& Barrier $(\times 10^{-2})$ $\downarrow$
& Retention $\uparrow$
& Leakage $\downarrow$
& $\lambda_{\max}^{\mathrm{new}}/\lambda_{\max}^{\mathrm{act}}$ $\downarrow$ \\
\midrule
Full ResNet-18
& $76.25 \pm 0.23$
& -- & -- & -- & -- \\
Fixed subspace, $S=5$
& $3.13 \pm 0.77$
& -- & -- & -- & -- \\
One-shot unfreeze, $S=5$
& $75.84 \pm 0.33$
& $0.17 \pm 0.29$
& $1.0000 \pm 0.0000$
& $0.0000 \pm 0.0000$
& $0.7913 \pm 0.0076$ \\
\bottomrule
\end{tabular}}
\end{table}
The fixed-subspace control fails on CIFAR-100, reaching only $3.13\%$ test
accuracy. This confirms that the initial constrained model does not contain
sufficient trainable capacity by itself. One-shot unfreezing, by contrast,
recovers most of the full-model performance and yields almost zero interpolation
barrier. However, it does not exhibit the same curvature signature as progressive
growth: its newly released directions have a restricted curvature ratio
$\lambda_{\max}^{\mathrm{new}}/\lambda_{\max}^{\mathrm{act}}
=0.7913\pm0.0076$.

Comparing this control with the progressive-growth runs in Table~\ref{tab:resnet_main}
shows that gradual relaxation produces substantially flatter newly released
directions. For example, validation-accuracy growth with $S=5$ has
$\lambda_{\max}^{\mathrm{new}}/\lambda_{\max}^{\mathrm{act}}
=0.201\pm0.014$, and $S=10$ further reduces this ratio to
$0.0896\pm0.0058$, both far below the one-shot value. Thus, the low-curvature
signal is not merely a consequence of starting from a smaller model or delaying
capacity release; it is tied to gradual constraint relaxation. At the same time,
Table~\ref{tab:resnet_main} shows that this curvature bias does not by itself
guarantee higher final accuracy. The same qualitative contrast also appears in a fully connected
overparameterized MLP on a hard two-moons task: progressive growth releases
directions with substantially lower restricted curvature than one-shot
unfreezing at comparable test accuracy. We report this architecture-agnostic
control in the Appendix~\ref{app:mlp_control_experiment}.
\section{Discussion and Limitations}
\label{sec:discussion}

Our results support a coherent picture across theory, toy models, and deep-network experiments. Progressive growth can be viewed as a sequence of constraint relaxations, and this mechanism induces a bias toward minima that remain compatible with many frozen-coordinate configurations. In the local theory, this compatibility is governed by an effective frozen-direction curvature through a Schur-complement reduction, together with a location term determined by the initialization. In the toy experiments, this prediction is directly visible: when basin locations are balanced, progressive growth increasingly concentrates selection on flatter minima, whereas in the energy--entropy trade-off regime the bias can reverse when sharper basins are sufficiently closer to initialization. In ResNet/CIFAR-100, the same mechanism does not appear as a gain in final accuracy, but it does leave a clear geometric footprint: moderate schedules preserve consecutive stage solutions, interpolation barriers remain small, and newly released directions are systematically flatter than already active ones.
\newline
A first important implication is conceptual. Progressive growth is best understood as inducing an \emph{optimization bias}, not as guaranteeing better predictive performance. The theory does not state that flatter minima should always generalize better, nor do our experiments support such a simplistic interpretation. Instead, it identifies a bias in \emph{which basins remain accessible} as constraints are progressively relaxed. This distinction is essential, especially in light of recent work showing that local sharpness metrics alone are not reliable universal predictors of generalization in modern deep learning~\cite{2023arXiv230207011A}. Our results are consistent with that view: the geometric signal induced by growth is strong, whereas its effect on final test accuracy is subtle and, in the deep-learning setting studied here, often negative.
Together, the internal controls in Table~\ref{tab:resnet_internal_controls}
and the progressive-growth runs in Table~\ref{tab:resnet_main} make this
distinction more explicit. One-shot unfreezing recovers most of the full-model
accuracy, showing that delayed capacity release alone can be sufficient for good
predictive performance. However, one-shot unfreezing does not produce the same
low-curvature newly released directions as progressive growth. Increasing the
number of progressive stages further strengthens this curvature bias in
Table~\ref{tab:resnet_main}, although its effect on transition barriers and final
accuracy remains schedule-dependent. Thus, progressive growth should be
understood as a mechanism for shaping the geometry of the training trajectory,
not as an unconditional recipe for improving accuracy.
\newline
A second implication is methodological. To study progressive growth, transition-level quantities appear more informative than final curvature scalars alone. These diagnostics are directly tied to the notion of accessibility under frozen constraints: they indicate whether a growth step preserves the current basin or instead pushes optimization toward a different region. In our deep-network experiments, these quantities reveal a coherent geometric structure even when final accuracy differences are small. This is also supported by our post-hoc analysis, where larger barrier and leakage correlate with larger endpoint disruption after expansion. More generally, this suggests that progressive expansion should be analyzed through the stability of its stage transitions, rather than only through end-of-training metrics.
The negative predictive result is itself informative. In the present ResNet setting, progressive growth should not be interpreted as a plug-in performance heuristic. Its main effect is geometric: it changes which regions remain accessible and how stable stage-to-stage transitions remain under expansion. This also suggests that any practical benefit of growth may depend less on staged expansion alone than on the ability to detect when successive expansions remain geometrically stable. In this sense, transition metrics such as barrier and leakage may eventually serve not only as post-hoc diagnostics, but also as online signals for adapting or stopping growth. A promising direction for future work is therefore to design adaptive growth rules that explicitly trade off transition stability against predictive performance.
\newline
At the same time, the present work has several limitations. First, the theory is local and relies on isolated nondegenerate minima with positive-definite Hessians. This stylized setting is useful for deriving the accessibility law, but modern deep networks are highly over-parameterized and often exhibit symmetries or near-degenerate solution sets. Extending the analysis to degenerate minima or minimum manifolds would therefore be a natural next step. Second, our deep-network study is intentionally focused: we consider a single architecture family and a single dataset. Although the geometric trends are clear, broader empirical coverage would be required before drawing strong conclusions about the universality of the effect across architectures, optimizers, normalization schemes, or larger-scale settings. Third, our empirical curvature estimates remain local proxies. They are useful for interpreting the released directions, but they do not by themselves define basin membership. This is precisely why we complement them with transition-based diagnostics, but it also means that the geometric picture remains approximate in high dimension.
A further limitation is practical. The current study does not optimize growth as an efficiency-oriented training strategy, and in our implementation it does not provide a clear advantage over full-model training in terms of final predictive performance. Our schedules are designed primarily to probe the theoretical mechanism, not to maximize accuracy or wall-clock efficiency. As a result, growth can degrade final performance when the schedule becomes too fragmented, as seen in the ResNet experiments. This is not a contradiction of the theory; rather, it highlights that accessibility bias and predictive performance are distinct objectives.

More broadly, the present perspective suggests a different way of thinking about model growth. Instead of viewing expansion only as a capacity-management heuristic, one can view it as a tool for \emph{shaping the optimization geometry}. In that interpretation, progressive growth becomes a mechanism for filtering out minima that are fragile under partial freezing and for favoring regions that are robust to successive constraint relaxations. This viewpoint could potentially connect growth to other forms of structured training, including sparse or masked optimization, pruning-and-regrowth strategies, curriculum over parameter subsets, and multi-stage warm-start procedures.
\section{Conclusion}
\label{sec:conclusion}

We studied progressive growth as a geometric intervention on optimization. Rather than merely changing model capacity over time, growth constrains training to a sequence of nested affine slices, thereby changing which minima remain accessible. Under local regularity assumptions, we showed that this accessibility follows an energy--entropy trade-off: minima are favored when they are both compatible with the initialization in the frozen coordinates and broad along those directions. In this sense, progressive growth induces a bias toward flatter minima, understood not as a universal predictor of generalization, but as a bias in accessibility under constraint relaxation.

The experiments support this view at two complementary levels. In controlled toy landscapes, progressive growth amplifies selection toward flat minima when basin locations are balanced, and exhibits the predicted reversal when sharper basins are sufficiently closer to initialization. In ResNet/CIFAR-100, the same mechanism appears through transition geometry rather than final accuracy: moderate schedules preserve consecutive stage solutions, interpolation barriers remain small, and newly released directions are systematically flatter than already active ones.
The internal controls further separate progressive growth from simpler
delayed-release strategies. One-shot unfreezing recovers most of the full-model
accuracy but does not exhibit the same low-curvature released directions, while
fixed-subspace training fails on CIFAR-100. This supports the view that the
geometric bias arises from gradual constraint relaxation rather than merely from
starting with a smaller submodel.
Overall, our results suggest that progressive growth should be studied not only as a way of allocating capacity, but also as a way of steering optimization. This perspective helps explain why staged expansion can substantially alter the regions explored by training even when its effect on final accuracy is limited, and points toward future training strategies that explicitly exploit accessibility under structured constraint relaxation.

\begin{credits}
\subsubsection*{\ackname}
This work was granted access to the HPC resources of IDRIS under the allocation A0191016927 made by GENCI. This work has received support from the French government, managed by the National Research Agency (ANR), under the France 2030 program with the reference “PR[AI]RIE-PSAI” (ANR-23-IACL-0008), "PEPR-SHARP" (ANR-23-PEIA-0008) and project "LLM4ALL" (ANR-23-IAS1-0008).

\subsubsection{\discintname}
The authors have no competing interests to declare that are
relevant to the content of this article.
\end{credits}
%
%
%
\bibliographystyle{src/splncs04}
\bibliography{biblio}

\newpage
\appendix
\section{Proofs for Section~2}
\label{app:theory_proofs}

In this appendix we prove the results stated in our theoretical analysis. Throughout, we fix a basin index $i$ and a stage $t$ when no confusion is possible. We write
\[
f_i^\star := f(\theta_i^\star),
\qquad
H_i := \nabla^2 f(\theta_i^\star),
\]
and denote by
\[
\mu_i := \lambda_{\min}(H_i),
\qquad
L_i := \lambda_{\max}(H_i)
\]
the extreme eigenvalues of the local Hessian. For $\alpha>0$, we also introduce the quadratic ellipsoid
\[
E_{i,\alpha}
:=
\left\{
\theta_i^\star+\delta : \frac12\,\delta^\top H_i\delta \le \alpha
\right\}.
\]

Let us remind the main assumptions and results below.
\begin{assumption}[Local regularity]
\label{ass:local}
For each minimum $\theta_i^\star$,
$$
\nabla f(\theta_i^\star)=0,
\qquad
H_i := \nabla^2 f(\theta_i^\star)\succ 0.
$$
Moreover, the Hessian is locally Lipschitz: there exist $r_i>0$ and $\rho_i>0$ such that, for all $\theta,\theta'\in B(\theta_i^\star,r_i)$,
$$
\|\nabla^2 f(\theta)-\nabla^2 f(\theta')\|
\le
\rho_i\|\theta-\theta'\|.
$$
\end{assumption}

\begin{proposition}[Local ellipsoidal approximation]
\label{prop:local-ellipsoid}
Under Assumption~\ref{ass:local}, there exist constants $c_{1,i},c_{2,i}>0$ and $\varepsilon_i^{\max}>0$ such that, for all $0<\varepsilon\le \varepsilon_i^{\max}$,
$$
\left\{
\theta_i^\star+\delta :
\frac12 \delta^\top H_i\delta \le c_{1,i}\varepsilon
\right\}
\subset
\mathcal{L}_{i,\varepsilon}
\subset
\left\{
\theta_i^\star+\delta :
\frac12 \delta^\top H_i\delta \le c_{2,i}\varepsilon
\right\}.
$$
\end{proposition}

\begin{theorem}[Effective curvature under frozen constraints]
\label{thm:schur}
Under Assumption~\ref{ass:local}, for $b$ in a neighborhood of $0$,
$$
\varphi_{i,t}(b)
=
\frac12\, b^\top \Sigma_{i,t} b + O(\|b\|^3),
$$
where
$$
\Sigma_{i,t}
=
H^{(t)}_{i,bb}
-
H^{(t)}_{i,ba}(H^{(t)}_{i,aa})^{-1}H^{(t)}_{i,ab}
$$
is the Schur complement of the Hessian block matrix in the coordinates $(a,b)$.
\end{theorem}

\begin{theorem}[Local accessibility law]
\label{thm:accessibility}
Under Assumption~\ref{ass:local}, assume that the frozen offset
$
b_{i,t}^0 = V_t^\top(\theta^0-\theta_i^\star)
$
has a continuous density $p_{i,t}$ in a neighborhood of $0$, conditional on $(U_t,V_t)$. Then, as $\varepsilon\to 0$,
$$
\mathbb{P}\!\left(\mathcal{A}_t\cap\mathcal{L}_{i,\varepsilon}\neq\emptyset \,\middle|\, U_t,V_t\right)
=
p_{i,t}(0)\,
\kappa_{p_t}(2\varepsilon)^{p_t/2}
\det(\Sigma_{i,t})^{-1/2}
\bigl(1+o(1)\bigr),
$$
where $\kappa_{p_t}$ denotes the volume of the unit ball in $\R^{p_t}$.
\end{theorem}

\subsection{A local Taylor estimate}

We begin with the standard second-order Taylor expansion with cubic remainder.

\begin{lemma}[Taylor expansion with cubic remainder]
\label{lem:taylor}
Under Assumption~\ref{ass:local}, for every $\delta$ such that $\|\delta\|\le r_i$,
\[
f(\theta_i^\star+\delta)
=
f_i^\star + \frac12\,\delta^\top H_i\delta + R_i(\delta),
\]
where
\[
|R_i(\delta)| \le \frac{\rho_i}{6}\|\delta\|^3 .
\]
\end{lemma}

\begin{proof}
Taylor's theorem with integral remainder gives
\[
f(\theta_i^\star+\delta)-f_i^\star
=
\frac12\,\delta^\top H_i\delta
+
\int_0^1 (1-s)\,\delta^\top\!\bigl(\nabla^2 f(\theta_i^\star+s\delta)-H_i\bigr)\delta\,ds.
\]
The remainder is therefore
\[
R_i(\delta)
=
\int_0^1 (1-s)\,\delta^\top\!\bigl(\nabla^2 f(\theta_i^\star+s\delta)-H_i\bigr)\delta\,ds.
\]
By Assumption~\ref{ass:local},
\[
\|\nabla^2 f(\theta_i^\star+s\delta)-H_i\|
\le \rho_i s\|\delta\|.
\]
Hence
\[
|R_i(\delta)|
\le
\int_0^1 (1-s)\,\|\delta\|^2 \cdot \rho_i s\|\delta\|\,ds
=
\frac{\rho_i}{6}\|\delta\|^3.
\]
\end{proof}

\subsection{Proof of Proposition~\ref{prop:local-ellipsoid}}

We first show that, for sufficiently small radii, the loss is comparable to its quadratic model.

\begin{proof}[Proof of Proposition~\ref{prop:local-ellipsoid}]
Choose
\[
\bar r_i := \min\!\left\{r_i,\frac{3\mu_i}{2\rho_i}\right\},
\]
with the convention that $3\mu_i/(2\rho_i)=+\infty$ when $\rho_i=0$. For every $\delta$ with $\|\delta\|\le \bar r_i$, Lemma~\ref{lem:taylor} gives
\[
f(\theta_i^\star+\delta)-f_i^\star
=
\frac12\,\delta^\top H_i\delta + R_i(\delta),
\qquad
|R_i(\delta)| \le \frac{\rho_i}{6}\|\delta\|^3.
\]
Since
\[
\delta^\top H_i\delta \ge \mu_i\|\delta\|^2,
\]
the choice of $\bar r_i$ implies
\[
\frac{\rho_i}{6}\|\delta\|^3
\le
\frac{\rho_i\bar r_i}{6}\|\delta\|^2
\le
\frac{\mu_i}{4}\|\delta\|^2
\le
\frac14\,\delta^\top H_i\delta.
\]
Therefore, for all $\|\delta\|\le \bar r_i$,
\begin{equation}
\label{eq:quadratic-comparison}
\frac14\,\delta^\top H_i\delta
\le
f(\theta_i^\star+\delta)-f_i^\star
\le
\frac34\,\delta^\top H_i\delta.
\end{equation}

Next, because $\theta_i^\star$ is a strict local minimum and the annulus
\[
\mathcal A_i := \{\theta_i^\star+\delta : \bar r_i \le \|\delta\| \le r_i\}
\]
is compact and does not contain $\theta_i^\star$, the continuous function
\[
\theta \mapsto f(\theta)-f_i^\star
\]
attains a strictly positive minimum on $\mathcal A_i$. Let
\[
\eta_i := \inf_{\bar r_i \le \|\delta\| \le r_i}\bigl(f(\theta_i^\star+\delta)-f_i^\star\bigr) > 0.
\]
Set
\[
\varepsilon_i^{\max} := \eta_i.
\]

Fix now $0<\varepsilon\le \varepsilon_i^{\max}$. If $\theta_i^\star+\delta \in \mathcal L_{i,\varepsilon}$, then by definition
\[
f(\theta_i^\star+\delta)-f_i^\star \le \varepsilon \le \eta_i.
\]
Hence necessarily $\|\delta\|<\bar r_i$, for otherwise the point would lie in the annulus where the loss gap is at least $\eta_i$. We may therefore apply \eqref{eq:quadratic-comparison}.

For the inner inclusion, suppose that
\[
\frac12\,\delta^\top H_i\delta \le \frac{2}{3}\varepsilon.
\]
Then
\[
f(\theta_i^\star+\delta)-f_i^\star
\le
\frac34\,\delta^\top H_i\delta
\le
\varepsilon,
\]
so $\theta_i^\star+\delta\in \mathcal L_{i,\varepsilon}$.

For the outer inclusion, suppose that $\theta_i^\star+\delta\in \mathcal L_{i,\varepsilon}$. Then
\[
f(\theta_i^\star+\delta)-f_i^\star \le \varepsilon,
\]
and \eqref{eq:quadratic-comparison} yields
\[
\frac14\,\delta^\top H_i\delta \le \varepsilon,
\]
that is,
\[
\frac12\,\delta^\top H_i\delta \le 2\varepsilon.
\]
Thus
\[
E_{i,\frac{2}{3}\varepsilon} \subset \mathcal L_{i,\varepsilon} \subset E_{i,2\varepsilon}.
\]
This proves the claim with $c_{1,i}=2/3$ and $c_{2,i}=2$.
\end{proof}

\subsection{Preliminaries for the reduced problem}

Let
\[
g(a,b) := f(\theta_i^\star + U_t a + V_t b),
\]
and write the Hessian of $g$ at $(a,b)=(0,0)$ in block form:
\[
\nabla^2 g(0,0)
=
\begin{pmatrix}
H_{aa}^{(t)} & H_{ab}^{(t)} \\
H_{ba}^{(t)} & H_{bb}^{(t)}
\end{pmatrix}.
\]
Because $[U_t\;V_t]$ is orthogonal,
\[
\nabla^2 g(0,0) = [U_t\;V_t]^\top H_i [U_t\;V_t],
\]
so it is symmetric positive definite whenever $H_i$ is. In particular, the principal block $H_{aa}^{(t)}$ is positive definite, and the Schur complement
\[
\Sigma_{i,t}
=
H_{bb}^{(t)} - H_{ba}^{(t)}(H_{aa}^{(t)})^{-1}H_{ab}^{(t)}
\]
is also positive definite.

The next lemma records the local Taylor expansion of $g$.

\begin{lemma}[Taylor expansion of the reduced coordinates]
\label{lem:g-expansion}
There exists a constant $C_{i,t}>0$ such that, for $(a,b)$ sufficiently close to $(0,0)$,
\[
g(a,b)-g(0,0)
=
\frac12 a^\top H_{aa}^{(t)} a
+
a^\top H_{ab}^{(t)} b
+
\frac12 b^\top H_{bb}^{(t)} b
+
R(a,b),
\]
with
\[
|R(a,b)| \le C_{i,t}\|(a,b)\|^3,
\]
and
\[
\nabla_a g(a,b)
=
H_{aa}^{(t)} a + H_{ab}^{(t)} b + r(a,b),
\qquad
\|r(a,b)\| \le C_{i,t}\|(a,b)\|^2.
\]
\end{lemma}

\begin{proof}
This is a direct consequence of Lemma~\ref{lem:taylor} applied to the composition
\[
(a,b)\mapsto \theta_i^\star + U_t a + V_t b.
\]
Since $[U_t\;V_t]$ is orthogonal, this map preserves Euclidean norms, and the local Lipschitz bound on the Hessian of $f$ transfers to $g$.
\end{proof}

\subsection{Proof of Theorem~\ref{thm:schur}}

\begin{proof}[Proof of Theorem~\ref{thm:schur}]
For notational simplicity, write
\[
H_{aa}:=H_{aa}^{(t)},\qquad
H_{ab}:=H_{ab}^{(t)},\qquad
H_{ba}:=H_{ba}^{(t)},\qquad
H_{bb}:=H_{bb}^{(t)}.
\]
Because $H_{aa}\succ 0$, the equation
\[
\nabla_a g(a,b)=0
\]
can be solved locally for $a$ as a function of $b$. By the implicit function theorem, there exists a neighborhood of $0$ and a unique $C^1$ map $a(b)$ such that
\[
\nabla_a g(a(b),b)=0,
\qquad
a(0)=0.
\]
Define
\[
\varphi_{i,t}(b)=g(a(b),b)-g(0,0).
\]

We first estimate $a(b)$. By Lemma~\ref{lem:g-expansion},
\[
0
=
\nabla_a g(a(b),b)
=
H_{aa}a(b)+H_{ab}b+r(a(b),b),
\]
with
\[
\|r(a(b),b)\|\le C\|(a(b),b)\|^2.
\]
Rearranging,
\[
a(b)
=
-H_{aa}^{-1}H_{ab}b - H_{aa}^{-1}r(a(b),b).
\]
Since $a(0)=0$ and $a$ is continuous, shrinking the neighborhood if necessary gives
\[
\|a(b)\| \le C_1\|b\|
\]
for $\|b\|$ small enough. Substituting this back into the previous display yields
\[
a(b) = -H_{aa}^{-1}H_{ab}b + O(\|b\|^2).
\]

Now expand $g(a(b),b)-g(0,0)$ using Lemma~\ref{lem:g-expansion}:
\[
\varphi_{i,t}(b)
=
\frac12 a(b)^\top H_{aa} a(b)
+
a(b)^\top H_{ab} b
+
\frac12 b^\top H_{bb} b
+
R(a(b),b).
\]
Because $a(b)=O(\|b\|)$, the remainder satisfies
\[
R(a(b),b)=O(\|b\|^3).
\]
Substituting
\[
a(b) = -H_{aa}^{-1}H_{ab}b + O(\|b\|^2)
\]
into the quadratic part gives
\[
\frac12 a(b)^\top H_{aa} a(b)
+
a(b)^\top H_{ab} b
+
\frac12 b^\top H_{bb} b
=
\frac12\, b^\top\!\bigl(H_{bb}-H_{ba}H_{aa}^{-1}H_{ab}\bigr)b + O(\|b\|^3).
\]
Therefore
\[
\varphi_{i,t}(b)
=
\frac12\, b^\top \Sigma_{i,t} b + O(\|b\|^3),
\]
with
\[
\Sigma_{i,t}=H_{bb}-H_{ba}H_{aa}^{-1}H_{ab}.
\]
This is exactly the claimed expansion.
\end{proof}

\subsection{A local comparison for the compatibility set}

We now compare the compatibility set
\[
\mathcal{B}_{i,t}(\varepsilon)
=
\{b : \varphi_{i,t}(b)\le \varepsilon\}
\]
to the ellipsoid associated with $\Sigma_{i,t}$.

\begin{lemma}[Ellipsoidal comparison for $\mathcal{B}_{i,t}(\varepsilon)$]
\label{lem:compatibility-ellipsoid}
Let
\[
\mathcal E_{i,t}(\alpha)
:=
\left\{
b : \frac12\,b^\top \Sigma_{i,t} b \le \alpha
\right\}.
\]
Then there exists $\varepsilon_0>0$ such that, for all sufficiently small $\varepsilon$,
\[
\mathcal E_{i,t}\bigl((1-C\sqrt{\varepsilon})\varepsilon\bigr)
\subset
\mathcal B_{i,t}(\varepsilon)
\subset
\mathcal E_{i,t}\bigl((1+C\sqrt{\varepsilon})\varepsilon\bigr)
\]
for some constant $C>0$ independent of $\varepsilon$.
\end{lemma}

\begin{proof}
By Theorem~\ref{thm:schur},
\[
\varphi_{i,t}(b)=\frac12\,b^\top\Sigma_{i,t}b + r(b),
\qquad
|r(b)|\le C_0\|b\|^3
\]
for $\|b\|$ sufficiently small. Let
\[
\underline\lambda_{i,t}:=\lambda_{\min}(\Sigma_{i,t})>0.
\]
Choose $\delta>0$ small enough that
\[
|r(b)|\le \frac14\,\underline\lambda_{i,t}\|b\|^2
\qquad\text{whenever }\|b\|\le \delta.
\]
Then, for $\|b\|\le\delta$,
\[
\varphi_{i,t}(b)
\ge
\frac12\,\underline\lambda_{i,t}\|b\|^2 - \frac14\,\underline\lambda_{i,t}\|b\|^2
=
\frac14\,\underline\lambda_{i,t}\|b\|^2.
\]
Since $\varphi_{i,t}(0)=0$ and $0$ is a strict local minimum, the compact sphere $\{b:\|b\|=\delta\}$ has strictly positive minimum value under $\varphi_{i,t}$. Hence, for $\varepsilon$ sufficiently small, every $b\in\mathcal B_{i,t}(\varepsilon)$ satisfies $\|b\|<\delta$, and therefore
\[
\|b\|
\le
2\sqrt{\varepsilon/\underline\lambda_{i,t}}.
\]
Consequently,
\[
|r(b)| \le C_1 \varepsilon^{3/2}
\qquad\text{for all }b\in\mathcal B_{i,t}(\varepsilon),
\]
for some constant $C_1>0$.

Now define $\eta_\varepsilon:=C_1\sqrt{\varepsilon}$. If
\[
b\in \mathcal E_{i,t}\bigl((1-\eta_\varepsilon)\varepsilon\bigr),
\]
then
\[
\frac12\,b^\top\Sigma_{i,t}b \le (1-\eta_\varepsilon)\varepsilon.
\]
Moreover $\|b\|=O(\sqrt{\varepsilon})$, so $|r(b)|\le \eta_\varepsilon\varepsilon$ for $\varepsilon$ small enough. Hence
\[
\varphi_{i,t}(b)
\le
(1-\eta_\varepsilon)\varepsilon + \eta_\varepsilon\varepsilon
=
\varepsilon,
\]
which proves the left inclusion.

Conversely, if $b\in\mathcal B_{i,t}(\varepsilon)$, then
\[
\frac12\,b^\top\Sigma_{i,t}b
\le
\varphi_{i,t}(b) + |r(b)|
\le
\varepsilon + \eta_\varepsilon\varepsilon
=
(1+\eta_\varepsilon)\varepsilon,
\]
which proves the right inclusion.
\end{proof}

\subsection{Proof of Theorem~\ref{thm:accessibility}}

\begin{proof}[Proof of Theorem~\ref{thm:accessibility}]
Fix the stage $t$ and condition on the chosen subspace $(U_t,V_t)$. In this conditional setting, the random variable $b_{i,t}^0$ has density $p_{i,t}$ in a neighborhood of $0$, and
\[
\mathbb P\!\left(\mathcal A_t\cap \mathcal L_{i,\varepsilon}\neq\emptyset \,\middle|\, U_t,V_t\right)
=
\mathbb P\!\left(b_{i,t}^0\in \mathcal B_{i,t}(\varepsilon)\,\middle|\, U_t,V_t\right)
=
\int_{\mathcal B_{i,t}(\varepsilon)} p_{i,t}(b)\,db.
\]

We first compute the volume of $\mathcal B_{i,t}(\varepsilon)$. By Lemma~\ref{lem:compatibility-ellipsoid},
\[
\mathcal E_{i,t}\bigl((1-C\sqrt{\varepsilon})\varepsilon\bigr)
\subset
\mathcal B_{i,t}(\varepsilon)
\subset
\mathcal E_{i,t}\bigl((1+C\sqrt{\varepsilon})\varepsilon\bigr).
\]
The exact volume of $\mathcal E_{i,t}(\alpha)$ is
\[
\Vol\bigl(\mathcal E_{i,t}(\alpha)\bigr)
=
\kappa_{p_t}(2\alpha)^{p_t/2}\det(\Sigma_{i,t})^{-1/2}.
\]
Therefore
\[
\Vol\bigl(\mathcal B_{i,t}(\varepsilon)\bigr)
=
\kappa_{p_t}(2\varepsilon)^{p_t/2}\det(\Sigma_{i,t})^{-1/2}\bigl(1+O(\sqrt{\varepsilon})\bigr).
\]

Next, because $\mathcal B_{i,t}(\varepsilon)$ shrinks to $\{0\}$ as $\varepsilon\to 0$ and $p_{i,t}$ is continuous at $0$,
\[
\sup_{b\in \mathcal B_{i,t}(\varepsilon)} |p_{i,t}(b)-p_{i,t}(0)| \to 0.
\]
Hence
\[
\int_{\mathcal B_{i,t}(\varepsilon)} p_{i,t}(b)\,db
=
p_{i,t}(0)\Vol\bigl(\mathcal B_{i,t}(\varepsilon)\bigr)
+
o\!\left(\Vol\bigl(\mathcal B_{i,t}(\varepsilon)\bigr)\right).
\]
Combining the two displays yields
\[
\mathbb P\!\left(\mathcal A_t\cap\mathcal L_{i,\varepsilon}\neq\emptyset \,\middle|\, U_t,V_t\right)
=
p_{i,t}(0)\,
\kappa_{p_t}(2\varepsilon)^{p_t/2}
\det(\Sigma_{i,t})^{-1/2}
\bigl(1+o(1)\bigr),
\]
which is the claimed asymptotic law.
\end{proof}

\subsection{Remark on the Gaussian case}

If, conditional on $(U_t,V_t)$, the initialization is isotropic Gaussian,
\[
\theta^0 \sim \mathcal N(0,\sigma^2 I_d),
\]
then
\[
b_{i,t}^0 = V_t^\top(\theta^0-\theta_i^\star)
\sim
\mathcal N(-V_t^\top\theta_i^\star,\sigma^2 I_{p_t}),
\]
so
\[
p_{i,t}(0)
=
(2\pi\sigma^2)^{-p_t/2}
\exp\!\left(
-\frac{\|V_t^\top\theta_i^\star\|^2}{2\sigma^2}
\right).
\]
The accessibility law from Theorem~\ref{thm:accessibility} then becomes
\[
\mathbb P\!\left(\mathcal A_t\cap\mathcal L_{i,\varepsilon}\neq\emptyset \,\middle|\, U_t,V_t\right)
=
(2\pi\sigma^2)^{-p_t/2}
\exp\!\left(
-\frac{\|V_t^\top\theta_i^\star\|^2}{2\sigma^2}
\right)
\kappa_{p_t}(2\varepsilon)^{p_t/2}
\det(\Sigma_{i,t})^{-1/2}
\bigl(1+o(1)\bigr).
\]
This makes explicit the energy--entropy tradeoff discussed in the main text: minima are favored when they are both close to the initialization in the frozen coordinates and broad along those directions.

\section{Additional Experimental Details}
\label{app:exp_details}

\subsection{Transition metrics}
\label{app:transition_metrics}

For two consecutive stage solutions $\theta^{(t)}$ and $\theta^{(t+1)}$, we consider the linear interpolation
\[
\theta_\alpha = (1-\alpha)\theta^{(t)} + \alpha \theta^{(t+1)}, 
\qquad \alpha \in [0,1].
\]
Let $\ell(\theta)$ denote the evaluation loss used for transition diagnostics. In our experiments, $\ell$ is the cross-entropy loss evaluated on the validation set, using a fixed number of batches for efficiency.

The \emph{interpolation barrier} is defined as the maximal excess loss along this path relative to the worse endpoint,
\[
\mathrm{Barrier}(\theta^{(t)},\theta^{(t+1)})
=
\max_{\alpha\in[0,1]} \ell(\theta_\alpha)
-
\max\{\ell(\theta^{(t)}),\ell(\theta^{(t+1)})\}.
\]
By construction, the barrier is nonnegative.

We also consider the signed loss difference between the two endpoints,
\[
\mathrm{EndpointGap}(\theta^{(t)},\theta^{(t+1)})
=
\ell(\theta^{(t+1)}) - \ell(\theta^{(t)}),
\]
and its absolute value,
\[
\mathrm{EndpointAbsGap}(\theta^{(t)},\theta^{(t+1)})
=
\bigl|\ell(\theta^{(t+1)}) - \ell(\theta^{(t)})\bigr|.
\]
For interpretability, we additionally report the positive and negative parts:
\[
\mathrm{EndpointImprovement}
=
\max\{0,\ell(\theta^{(t)})-\ell(\theta^{(t+1)})\},
\]
\[
\mathrm{EndpointDegradation}
=
\max\{0,\ell(\theta^{(t+1)})-\ell(\theta^{(t)})\}.
\]

Retention is a binary indicator of whether the transition remains in a low-barrier regime. Given a fixed threshold $\tau>0$, we define
\[
\mathrm{Retention}(\theta^{(t)},\theta^{(t+1)})
=
\mathbf{1}\!\left\{
\mathrm{Barrier}(\theta^{(t)},\theta^{(t+1)}) \le \tau
\right\}.
\]
In all deep-learning experiments, we use $\tau=0.1$.

Leakage is defined as the complement of retention,
\[
\mathrm{Leakage}(\theta^{(t)},\theta^{(t+1)})
=
1-\mathrm{Retention}(\theta^{(t)},\theta^{(t+1)}).
\]
Thus, leakage is also binary: it indicates whether a transition leaves the retained low-barrier regime.

For each transition, we also record the full interpolation profile, the maximum loss attained along the path, and the value of $\alpha$ at which this maximum occurs.

\subsection{Restricted curvature metrics}
\label{app:curvature_metrics}

To characterize the local geometry of stage transitions, we estimate Hessian quantities on restricted parameter subspaces. Given a binary mask $m\in\{0,1\}^d$, we define masked Hessian-vector products by projecting both the input direction and the output onto the masked coordinates.

For each post-expansion solution, we estimate the top restricted Hessian eigenvalue $\lambda_{\max}$ using power iteration;
the restricted trace using a Hutchinson estimator;
and stochastic log-determinant estimates through a Lanczos-based approximation when needed.

The main paper reports the top eigenvalue and trace-based quantities.
These quantities are computed on two masks, namely
 the \emph{active subspace}, i.e., the parameters already trainable at the current stage;
 and the \emph{newly released subspace}, i.e., the parameters that become trainable when moving from stage $t$ to stage $t+1$.
We also report the dimensions of these masks and, when useful, the trace normalized by the number of active parameters.

\subsection{CIFAR-100 / ResNet-18 setup}
\label{app:resnet_setup}

All deep-learning experiments are conducted on CIFAR-100. We use the standard train/test split and further divide the original training set into train and validation subsets using a fixed 90/10 split. This yields approximately $45{,}000$ training examples and $5{,}000$ validation examples, while the test set contains the standard $10{,}000$ examples.
Training data use standard CIFAR-style augmentation: random crop with padding $4$, random horizontal flip, tensor conversion, and channel-wise normalization. Validation and test data use only tensor conversion and normalization. The normalization constants are the standard CIFAR-100 statistics
\(
\mu=(0.5071, 0.4867, 0.4408),
\;
\sigma=(0.2675, 0.2565, 0.2761).
\)

Our default architecture is a CIFAR-style ResNet-18. It consists of a $3\times 3$ convolutional stem with stride $1$, followed by four residual stages with block counts $[2,2,2,2]$, channel widths $64,128,256,512$, global average pooling, and a final linear classifier. 

Optimization uses SGD with Nesterov momentum. The main hyperparameters are:
 learning rate: $0.1$,
 momentum: $0.9$,
 weight decay: $5\times 10^{-4}$,
label smoothing: $0$,
total number of epochs: $100$.
We use cosine annealing over the full training horizon. Automatic mixed precision is enabled. The training batch size is $128$, while evaluation uses batch size $256$.

The main ResNet experiments compare a full-model baseline with all parameters trainable from the start; progressive growth schedules with $S\in\{3,5,10\}$ stages; with two growth triggers: validation accuracy and training loss.
Unless stated otherwise, the patience parameter is set to $2$ epochs. The main ResNet results are averaged over $3$ random seeds.

\subsection{Progressive growth mechanism}
\label{app:growth_impl}

Our practical growth mechanism is implemented by assigning each trainable tensor to a sequence of stage groups. For convolutional and linear layers, this grouping follows the first tensor dimension, which corresponds to output channels or output units. Roughly speaking, stage $t$ activates approximately $(t+1)/S$ of the coordinates, up to discrete effects due to tensor shape.

The final classifier layer is kept active from the first stage. Batch-normalization parameters and running statistics are grouped consistently with the corresponding feature groups whenever possible, so that normalization variables are released together with the channels they control. Parameters not covered by these structured cases are assigned stage groups at random.

At each optimization step, frozen parameters are clamped to their initialization values and their gradients are masked to zero. The same masking is applied to optimizer states of matching shape. Frozen floating-point buffers, including batch-normalization running statistics, are likewise restored to their initial values. Operationally, this implements the affine-slice interpretation used in the theory: at each stage, optimization is restricted to the active coordinates while the frozen complement remains fixed at initialization.

\subsection{Internal control protocols}
The full-model baseline trains the same architecture with all parameters active
from initialization. Fixed-subspace training uses the same initial active
subspace as progressive growth, but never releases the frozen complement.
One-shot unfreezing also starts from the same initial constrained submodel, but
releases all remaining parameters in a single transition. Progressive growth
releases the same total set of parameters over multiple transitions. These
controls separate three effects: training a restricted submodel, delaying
capacity release, and gradually relaxing constraints.

\subsection{Algorithmic summary}
The practical procedure used in the deep-learning experiments can be summarized
as follows.

\begin{enumerate}
    \item Initialize the full network parameters $\theta^0$.
    \item Partition trainable tensors into $S$ nested stage groups. For
    convolutional and linear layers, groups are formed along the output-channel
    or output-unit dimension. The classifier is active from the first stage.
    \item Activate the first group and freeze all remaining parameters at their
    initialization values.
    \item Train the active parameters using the chosen optimizer while clamping
    frozen parameters, frozen buffers, gradients, and matching optimizer states
    to their frozen values.
    \item Monitor the growth criterion, either validation accuracy or training
    loss.
    \item When the criterion has not improved for the specified patience, save
    the pre-expansion solution, release the next parameter group, and continue
    optimization from the current weights.
    \item After each expansion, compute transition diagnostics between the
    pre-expansion and post-expansion stage solutions: interpolation barrier,
    retention, leakage, endpoint gap, and restricted curvature on the active and
    newly released subspaces.
    \item Repeat until all groups have been released. At test time, the final
    model is the full architecture.
\end{enumerate}

\subsection{Stage selection and transition evaluation}
\label{app:stage_selection}

For each stage, we track a stage-dependent best checkpoint according to the chosen trigger:
\begin{itemize}
    \item for the \texttt{val\_acc} trigger, the stage metric is validation accuracy;
    \item for the \texttt{train\_loss} trigger, the stage metric is the negative training loss.
\end{itemize}
If the stage metric does not improve for a number of epochs equal to the patience parameter, growth advances to the next stage. In the main experiments, stage transitions are determined by this patience-based rule. The internal control experiments with validation-accuracy growth use the same
patience-based rule, without forcing stage transitions at predetermined epochs. The best validation checkpoint over the full run is used for final reporting.

When a stage transition occurs, we compare:
\begin{itemize}
    \item the pre-expansion checkpoint, i.e., the model immediately before releasing new parameters;
    \item the best checkpoint obtained after re-optimizing in the expanded subspace.
\end{itemize}
All transition metrics are computed between these two endpoints.

For interpolation-based diagnostics, we evaluate the loss on $11$ equally spaced points along the segment between the two endpoints. To reduce cost, transition losses are computed on at most $2$ batches of the validation loader. The same transition protocol is used throughout the reported deep-learning experiments.

\subsection{Hessian estimation details}
\label{app:hessian_details}

Restricted Hessian metrics are computed using a single batch drawn from a dedicated non-augmented training loader. The Hessian batch size is $128$. In the main configuration, the estimators use $10$ power-iteration steps for $\lambda_{\max}$,
 $4$ Hutchinson probes for the trace,
 $2$ stochastic Lanczos probes for the log-determinant,
 $8$ Lanczos steps,
damping $10^{-3}$.
The main paper relies primarily on $\lambda_{\max}$ and trace-based quantities.

\subsection{Patience ablation}
\label{app:patience_ablation_setup}

For the aggressive $S=10$ regime, we additionally study the effect of patience on stage stability. In the training-loss-triggered experiments, we evaluate patience values $p\in\{1,2,5\}$, while keeping all other optimization and data settings identical to the main ResNet experiments.

\subsection{MLP control experiment}
\label{app:mlp_control_experiment}

To test whether the curvature effect observed in ResNet is specific to
convolutional residual architectures, we run a controlled experiment with an
overparameterized fully connected MLP on a hard two-moons classification task.
The dataset contains $1200$ training examples and $3000$ test examples, with
noise level $0.30$. The MLP has four hidden layers with width $256$. We compare
full-model training, fixed-subspace training, one-shot unfreezing, and
progressive growth.

All methods reach comparable test accuracy, around $91.5\%$--$91.7\%$,
indicating that the growth procedure does not materially affect predictive
performance in this simple setting. Interpolation barriers are also nearly zero
for all growth variants, suggesting that the task does not induce strongly
separated stage-wise solutions. However, the restricted-curvature signal is
clear: one-shot unfreezing yields
$\lambda_{\max}^{\mathrm{new}}/\lambda_{\max}^{\mathrm{act}}
=0.5733\pm0.1552$, whereas progressive growth yields
$0.1691\pm0.0938$ for $S=5$ and $0.1142\pm0.0329$ for $S=10$. Thus, even in a
fully connected overparameterized model, progressive growth releases directions
that are substantially flatter than the already active subspace. This supports the
interpretation that the growth-induced curvature bias is not specific to residual
convolutional networks.

\begin{table}[!htbp]
\centering
\caption{
Controlled MLP experiment on a two-moons classification task.
All methods achieve comparable test accuracy and nearly zero interpolation
barriers, reflecting the simplicity of the task. Nevertheless, progressive growth
releases directions with substantially lower restricted curvature than one-shot
unfreezing, suggesting that the curvature effect is not specific to convolutional
residual architectures.
}
\label{tab:mlp_harder_twomoons_controls}
\resizebox{\linewidth}{!}{
\begin{tabular}{lccccc}
\toprule
Method
& Test acc. (\%) $\uparrow$
& Barrier $(\times 10^{-2})$ $\downarrow$
& Retention $\uparrow$
& Leakage $\downarrow$
& $\lambda_{\max}^{\mathrm{new}}/\lambda_{\max}^{\mathrm{act}}$ $\downarrow$ \\
\midrule
Full MLP
& $91.58 \pm 0.16$
& -- & -- & -- & -- \\
Fixed subspace, $S=5$
& $91.50 \pm 0.37$
& -- & -- & -- & -- \\
One-shot unfreeze, $S=5$
& $91.72 \pm 0.14$
& $0.00 \pm 0.00$
& $1.0000 \pm 0.0000$
& $0.0000 \pm 0.0000$
& $0.5733 \pm 0.1552$ \\
Progressive growth, $S=5$
& $91.60 \pm 0.12$
& $0.02 \pm 0.03$
& $1.0000 \pm 0.0000$
& $0.0000 \pm 0.0000$
& $0.1691 \pm 0.0938$ \\
Progressive growth, $S=10$
& $91.53 \pm 0.14$
& $0.00 \pm 0.00$
& $1.0000 \pm 0.0000$
& $0.0000 \pm 0.0000$
& $0.1142 \pm 0.0329$ \\
\bottomrule
\end{tabular}}
\end{table}

\subsection{Schedule sensitivity.}
Forced expansion schedules preserve the qualitative curvature contrast between
one-shot and progressive growth, but can degrade final accuracy and increase
interpolation barriers. This confirms that the timing of constraint relaxation
matters: the curvature bias induced by progressive growth is robust
qualitatively, but its effect on transition stability and predictive performance
is schedule-dependent.

\section*{Disclosure of Generative AI Use}
A generative AI tool was used in a limited manner to improve the readability of the manuscript, including grammar correction and stylistic editing. It was not used to generate scientific claims, results, proofs, experiments, figures, or references. All technical content, interpretations, and conclusions were produced and verified by the authors, who take full responsibility for the final manuscript.

\end{document}